\documentclass[11pt]{article}
\usepackage{amssymb}
\usepackage{amsmath}
\usepackage{amsthm}
\usepackage[T1]{fontenc}
\usepackage{stmaryrd}
\usepackage{natbib}

\newcommand{\Km}{\sqrt{\frac{n}{\log m}}}
\newcommand{\Kmm}{\sqrt{\frac{\log m}{n}}}
\newcommand{\Pro}{\mathbb P}
\newcommand{\Ex}{\mathbf E}
\newcommand{\Var}{\mathbf{Var}}
\newcommand{\ths}{\theta^*}
\newcommand{\tth}{\tilde{\theta}}

\newcommand{\hth}{\hat{\theta}}
\newcommand{\sign}{{\rm sign}}
\newcommand{\indyk}{\mathbb{I}}

\newcommand{\mX}{\mathcal{X}}
\newcommand{\mY}{\mathcal{Y}}
\newcommand{\mF}{\mathcal{F}}
\newcommand{\mG}{\mathcal{G}}
\newcommand{\mZ}{\mathcal{Z}}
\newcommand{\bF}{\mathbf{F}}
\newcommand{\hlamn}{\hat{\lambda}_n}
\newcommand{\lamn}{\lambda_n}
\newcommand{\fthe}{f_\theta}
\newcommand{\fths}{f_{\theta^*}}

\newcommand{\fhth}{f_{\hat{\theta}}}
\newcommand{\ftth}{f_{\tilde{\theta}}}
\newcommand{\hC}{\hat{C}}
\newcommand{\tphif}{\tilde{\phi}_f}

\newcommand{\tpfth}{\tilde{\phi}_{f_\theta}}

\newtheorem{lemma}{Lemma}
\newtheorem{theorem}{Theorem}

\newtheorem{definition}{Definition}

\title{Oracle inequalities for ranking and $U$-processes with Lasso penalty}

\author{Wojciech Rejchel\\
{\small Faculty of Mathematics and Computer Science}\\
{\small Nicolaus Copernicus University}\\
{\small   Chopina 12/18, 87-100 Toru\'n, Poland}\\
{\small Faculty of Mathematics, Informatics and Mechanics}\\
{\small University of Warsaw}\\
{\small   Banacha 2, 02-097 Warsaw, Poland}\\
{\small  wrejchel@gmail.com}}

\date{}
\begin{document}
\maketitle

\begin{abstract}
We investigate properties of estimators obtained by minimization of $U$-processes with 
the Lasso penalty in high-dimensional settings.  Our attention is focused on the ranking problem that is popular in machine learning. It is related to  
guessing the ordering between objects on the basis of their observed predictors.
We prove the oracle inequality for the excess risk of the considered estimator as well as
the bound for the $l_1$ distance $|\hth - \ths|_1$ between the estimator and the oracle.
\end{abstract}


\vspace{0.3 cm}

\section{Introduction}
\label{intro}

Model selection is an important challenge, if one works with data sets containing many predictors. Finding relevant variables helps to understand better 
the problem and improves statistical inference. In the literature there are many methods 
solving such problems. One of them is empirical risk minimization with the penalty, for instance Lasso \citep{tib:96}.
The main characteristic of this procedure is an ability to select
relevant predictors and estimate unknown parameters simultaneously. In the paper we apply these ideas to the pairwise ranking problem (ordinal regression) 
that relates to predicting or guessing the ordering between objects on the basis of their observed predictors.  The problem of ranking has numerous applications in practice, for instance in information retrieval, banking or quality control. This is one of the reasons why it has been extensively studied recently \citep{fss:04,cortesmoh:04, aghhr:05, cz:06, rudin:06, clv:08, rej:12, chenwu:13, lai:13, laporte:14}.
 The ranking  problem is one of the examples of using $U$-processes
in machine learning.
The goal of the paper is to investigate properties of estimators obtained by minimization of 
$U$-processes with the Lasso penalty. We study the quality of such estimators in prediction  as well as model selection. We focus our attention on ranking estimators, but our results can be easily generalized as we will explain below. We start the paper with describing the statistical framework of the ranking problem.

Let us consider two objects that are randomly selected from a population. We assume that they are described by  a pair of
independent and identically distributed (with respect to the measure $P$) random vectors
$Z=(X,Y)$ and
$Z'=(X',Y')$ taking values in $\mathcal{X}
\times \mY,$ where $\mathcal{X} \subset \mathbb{R}^d$ and $\mY \subset \mathbb{R}. $ Random vectors $X$ and $X'$
are regarded as predictors, while $Y$ and $Y'$ are unknown variables
that define the ordering between objects. Namely, the object $z$ is ''better'' (faster, stronger etc.) than the object $z',$ if $y >y'.$
Our task is to predict the ordering between objects only on the basis of observations $X, X'.$ To do it we construct functions $f : \mathcal{X} \times \mathcal{X} \rightarrow \mathbb{R}$, called ranking rules,
that predict the ordering between objects
in the following way:
$$
\textrm{if }\; f (x , x') > 0, \;
\textrm{then we predict that } \; y > y'.
$$

The natural approach is to look for a function $f$ which, for a fixed loss function
$\phi: \mathbb{R} \rightarrow [0, \infty),$ minimizes the theoretical  risk
\begin{equation}\label{w1}
Q(f) =  \Ex  \phi \left[{\rm sign }(Y-Y') f (X,X') \right]
\end{equation}
in some family of  ranking rules $\mathbf{F}$
and sign$(t)=1$ for $t>0,$ sign$(t)=-1$ for $t<0$ and sign$(t)=0$ for $t=0.$
We cannot calculate $f^0={\rm arg} \min\limits_{f \in \mathbf{F}} Q(f) $ directly, but if we possess a learning sample of independent and identically distributed (with respect to the measure $P$) random vectors $Z_1 =(X_1,Y_1), \ldots ,Z_n =(X_n, Y_n),$ then we can minimize the empirical analog of the risk \eqref{w1}.
In the ranking problem a natural candidate for the empirical risk is a $U$-statistic
\begin{equation}
\label{w2}
Q_n (f)= \frac{1}{n(n-1)} \sum_{i \neq j} \phi_f (Z_i, Z_j ),
\end{equation}
where $\phi_f (z,z') =\phi [\mathrm{sign} (y-y') \: f(x,x')].$
The sum in \eqref{w2} is taken over all pairs of distinct indices $(i,j) \in \{1, \ldots, n\}^2.$ Obviously, we could also divide a sample into two halves and consider the unbiased estimator of the risk \eqref{w1} as an average of independent random variables 
\begin{equation}
\label{ww1}
 \frac{1}{N} \sum_{i=1}^N
 \phi_f \left(Z_i, Z_{N+i} \right),
\end{equation}
where $N=\left\lfloor \frac{n}{2} \right\rfloor .$ 
We could also divide a sample into even and odd observations and so on. However, the $U$-statistic \eqref{w2} is an average over all such divisions, that is for every permutation $\pi$ of a set $\{1, \ldots, n\}$ we have
$$\frac{1}{n(n-1)} \sum_{i \neq j} \phi_f(Z_i, Z_j)   =
\frac{1}{n!} \sum_\pi \frac{1}{N} \sum_{i=1}^N
 \phi_f \left(Z_{\pi(i)}, Z_{\pi(N+i)} \right).$$ Thus, the esimator \eqref{w2} seems to be more ''stable'' than \eqref{ww1}. Indeed, provided that mild conditions are satisfied the estimator \eqref{w2} has least variance among unbiased estimators of \eqref{w1}, see
\citet{serf:80}.

The  family $\mF ,$ that the estimator of $f^0$ will come from, is a class of linear combinations of a given collection of functions. Namely, consider a finite (but usually very large) family of base functions  
$\{ \psi_1, \ldots, \psi_m\},$
where $\psi_k : \mX \times \mX \rightarrow \mathbb{R}, k=1, \ldots m.$ We focus on the family
\begin{equation}
\label{FF}
\mathcal{F} = \left\{ f_\theta (x, x')= \sum_{k=1}^m \theta_k \psi_k (x, x'):
\theta \in \Theta \right\} \subset \bF,
\end{equation}
and the set $\Theta$ is a convex subset of $\mathbb{R}^m.$
If $m=d$ and $\psi_k (x,x')= x_k -x'_k, $ then $\mF$ is a class of linear ranking rules.
In this case we can think that the number of predictors $d$ that we observe is huge, but only a few of them are significant in the model. In this sense we say that the model is ''sparse''.

Furthermore, in high-dimensional problems one usually adds the penalty to the empirical risk \eqref{w2}. In the paper we consider the Lasso penalty, so we are to minimize, over the family \eqref{FF}, a function 
\begin{equation}
\label{w3}
 Q_n(f_\theta) + \hlamn |\theta |_1,
\end{equation}
where $|\cdot|_1$ is the $l_1$-norm of a vector $\theta$, i.e. $ | \theta |_1= \sum\limits_{k=1}^m |\theta_k|,$ and  $\hlamn>0$ is a smoothing parameter that will be estimated from the data, as it will be explained later.
This parameter is a balance between minimizing the empirical risk \eqref{w2}
and the penalty. We denote the minimizer of \eqref{w3} by $f_{\hth}.$
The form of the penalty is crucial,
because its singularity at the origin implies that some coordinates of the minimizer are
exactly  zeros, if $\hlamn$ is sufficiently large. Thus, by minimizing the function \eqref{w3} we simultaneously select significant parameters in the model and estimate them. 
Finally, we assume that the model is ''sparse'', that is using only few base functions we can 
appropriately approximate the best function $f^0 $ with respect to the risk.

The $0-1$ loss function $\phi(t) = \mathbb{I}_{(- \infty, 0)} (t)$ seems to be the
most adequate, and  the theoretical  risk \eqref{w1} becomes  probability of incorrect
ranking. However, in this case the empirical risk \eqref{w2} is discontinuous and minimization of the function \eqref{w3} is computationally difficult. Thus, we assume that the loss fuction $\phi$ is convex which makes the function \eqref{w3} convex with respect to  $\theta.$ 

In the literature one can find many papers concerning properties of Lasso estimators. They investigate the risk of estimators  as well as  model selection ability \citep{ greenshtein:06, geertar:06, zhaoyu:06, meinbul:06, buneatw:07, geer:08,bulgeer:11, nega:12, huang:12}. 
However, these papers study the case that one minimizes an empirical process (the sum of independent random variables) with the Lasso penalty. In the current paper we are to minimize \eqref{w3}, that is a $U$-process with the Lasso penalty, and we want to study properties of estimators based on such minimization. Obviously, $U$-processes are more complex than empirical processes, in fact they are generalization of empirical processes. Characteristics of ranking estimators with the penalty were investigated, among others, in \citet{clv:08, rej:12,chenwu:13}, but these results consider only their quality 
in predicting the ordering between objects.
In many practical problems (in genetics or biology) finding a (small) subset of significant predictors is equally as important (or even more) as predicting accurately  the ordering
between objects. Therefore, in the current paper we study both prediction and model selection quality of estimators obtained by minimization of \eqref{w3}. The practical performance of such estimators was studied, among others, in \citet{lai:13, laporte:14} on the basis of popular data sets. Thus, we focus on theoretical properties of ranking estimators and  
 prove similar theorems to those  contained in  papers that are cited at the beginning of this paragraph. 

In our main result (Theorem \ref{main}) we consider the excess risk of the estimator, that 
is $Q(\fhth) -Q(f^0) .$ It is a difference between the risk of the estimator and the risk of the best ranking rule in the class $\bF .$ We show that the estimator $\fhth$ behaves almost like the ''oracle'' that knows which parameters (base functions) should be included in the 
model by balancing the approximation risk and sparseness.
Moreover, in the same theorem we show the inequality for the $l_1$-distance $|\hth - \ths|_1$ between the estimator and the oracle that concerns model selection quality of the procedure.
Our results are nonasymptotic and allow the number of parameters $m$ to grow polynomially 
fast with the sample size $n.$
Theorem \ref{main} is a
strict analog of \citet[Theorem 2.2]{geertar:06} or \citet[Theorem 6.4]{bulgeer:11}, but
it relates to $U$-processes with the Lasso penalty instead of empirical processes. 

We have already mentioned that  the ranking problem with the Lasso penalty is contained in the general problem of minimizing a
$U$-process with the Lasso penalty.  $U$-processes are  used in many problems of statistical learning, for instance in generalized 
regression models or survival analysis \citep{han:87, sherman:93, kleinsp:92,pg:99, bobluk:12}. Properties of ranking estimators with the Lasso penalty that are described in 
Theorem \ref{main} can be generalized to $U$-processes with the Lasso penalty. Indeed, for a function $f: \mX \times 
\mX \rightarrow \mathbb{R}$ we define its risk as 
$$
\tilde{Q}(f) = \Ex \; \tilde{\phi} \left[f(X,X'),Y,Y' \right],
$$
where $\tilde{\phi} : \mathbb{R} \times \mY \times \mY \rightarrow \mathbb{R} $ is the loss function and $Z=(X,Y), Z'=(X',Y') \in \mX \times 
\mY$ are independent random vectors with the same distribution $P.$ Moreover, assume that 
random vectors $Z_1 =(X_1,Y_1), \ldots ,Z_n =(X_n, Y_n)$ are independent copies of them.
Define families $\bF, \mF$  as above and introduce a notation $\tphif (z,z') =\tilde{\phi} [ f(x,x'), y, y'].$
Then the unknown function $f^0={\rm arg} \min\limits_{f \in \mathbf{F}} \tilde{Q}(f) $  can be estimated by the minimizer $\fhth$ of a $U$-process with the Lasso penalty
\begin{equation}
\label{w4}
\tilde{Q}_n(\fthe) + \hlamn |\theta |_1,
\end{equation}
where
$$
\tilde{Q}_n(\fthe)= \frac{1}{n(n-1)} \sum_{i \neq j } \tpfth(Z_i,Z_j).
$$
Results stated for ranking estimators in the rest of the paper can be easily extended to 
minimizers of \eqref{w4}.

The paper is organized as follows: in Section~\ref{AandM} we state the main theorem of the paper  and discuss it. Its assumptions are elaborated in Section \ref{discussion}. The proof of Theorem \ref{main} is divided into two parts that 
are contained in Appendices \ref{stochar} and \ref{detar}. We conclude the paper in Section 
\ref{conclusions}.

\section{Assumptions and main result}
\label{AandM}

We start this section with presenting and briefly discussing conditions that are needed to obtain the results. We explain them more precisely in the next section.

{\bf Assumption 1.} The loss function $\phi$ is convex, besides it satisfies the Lipschitz condition with the constant $L>0$, that is 
\begin{equation}
\label{lip}
|\phi_{f_\theta} (z,z') -  \phi_{f_{\theta '}} (z,z')| \leq L \, |f_\theta (x,x') - 
f_{\theta '} (x,x')|
\end{equation}
for every $\theta, \theta ' \in \Theta$ and $z=(x,y) , z'=(x',y') \in \mX \times \mY.$ 
We have already mentioned that convexity of the loss function is needed to make the procedure 
computationally effective. However, it is also substantially used in the proof of
Theorem \ref{main}, for instance 
''to localize the problem'' (we can investigate only the neighbourhood of the oracle).

{\bf Assumption 2.} Let $\bF$ be a family of functions with a pseudonorm $|| \cdot ||.$ There exists a strictly convex function 
$G:[0, \infty) \rightarrow [0, \infty)$ with $G(0) = 0,$ such that for every $f_\theta \in
\mF$ 
 we have
\begin{equation}
\label{margin}
Q(\fthe) - Q(f^0) \geq G\left( || \fthe -f^0|| \right).
\end{equation}
    
Similar conditions that uniformly bound the excess risk  are often called  ''margin conditions'' in the literature \citep{mamtsyb:99, geer:08, bulgeer:11}. They are necessary to obtain 
''fast rates'' for estimators based on empirical processes \citep{m:02, bbm:05,bja:06} as well as $U$-processes \citep{clv:08, rej:12}. Fast rates refer to probabilistic inqualities that bound the estimation risk  by the expression depending on the sample size 
$n$ like $n^{-\beta}$ with $\beta > \frac{1}{2}\:.$ Furthermore, notice that Assumption 2 
(as well as Assumption 3, below) depends on the choice of the pseudonorm $||\cdot ||.$

For the function $G$ given in Assumption 2 we define its 
convex conjugate as a function $H: [0, \infty) \rightarrow [0,\infty) $ given as
$$
H(v)= \sup_{u \geq 0} \{uv - G(u)\}
$$
for every $v \geq 0.$ For instance, if $G(u) = a u^{\frac{2}{\alpha}}$ for a positive constant $a$ and $\alpha \in (0,1],$ then $H(v)= \frac{2-\alpha}{2}
\left(
\frac{\alpha }{2a}\right)^{\frac{\alpha}{2-\alpha}} \, v^{\frac{2}{2-\alpha}}  \:. $ The function $H$ will be involved in  a bound for the estimation risk of the estimator and will determine the rate of the oracle inequality.

Before we state the next assumption we need a few notations and a definition. 
Let $S$ be a subset of $\{1,\ldots, m\}.$ We denote by $|S|$ the number of elements of the 
set $S$ and its complement by $S'= \{1,\ldots, m\} \backslash S.$ Moreover, let $\theta_S$ be the vector that  equals  $\theta$ on the set $S$ and has zeros elsewhere, i.e.
$(\theta_S)_k = \theta_k \; \indyk (k \in S),  k=1, \ldots, m.$
Thus, we have $\theta= \theta_S + \theta_{S'}.$

\begin{definition}
The compatibility condition is fulfilled for the set $S$ with a constant $A(S)$, if 
for every $\theta \in \mathbb{R}^m$ satisfying $|\theta_{S'}|_1 \leq 3 |\theta_{S}|_1$ the 
following inequality holds
\begin{equation}
\label{cc}
|\theta_{S}|_1 \leq \frac{||\fthe|| \, \sqrt{|S|} }{A(S)} \: .
\end{equation}
\end{definition}
The above definition is borrowed from \citet[Section 6.7]{bulgeer:11}. It constitutes the relation between the pseudonorm $|| \cdot ||$ and the norm $|\cdot|_1.$ 
The compatibility condition or similar assumptions such as the coherence condition or the restricted eigenvalue condition are standard in high-dimensional problems \citep{buneatw:07, geer:08, bickel:09, bulgeer:11}.

{\bf Assumption 3.} The set $\Theta_1$ that minimization in \eqref{oracle}, given below,   is taken over consists of 
vectors $\theta$ such that the compatibility condition is satisfied for 
$S_\theta = \{ 1\leq k \leq m: \theta_k \neq 0\} .$ 
We will call $S_\theta$ the support of the vector $\theta.$

The last technical assumption helps us to handle the stochastic part of the proof of the main result.

{\bf Assumption 4.} Suppose that 
\begin{equation*}
\max_{1 \leq k \leq m} \left| \psi_k \right|_\infty \leq \sqrt{\frac{n}{\log m}} \:,
\end{equation*}
where $|g|_\infty = \sup_{x,x'} |g(x,x')|.$

Moreover, we assume that the value
$$
C^2= \max_{1 \leq k \leq m} \Ex \psi_k^2 (X,X')
$$
is finite. We do not need to know  the value of $C^2,$ the estimator $\hC ^2$ of it 
will be sufficient in further 
argumentation. We set $\hC^2$ to be the empirical version of $C^2$, namely  it is 
defined by the following formula 
$$
\hC^2= \max_{1 \leq k \leq m} \;\frac{1}{n(n-1)} \sum_{i \neq j} \psi_k^2 (X_i,X_j).
$$
Having the estimator $\hC ^2$ we can take the smoothing parameter $\hlamn$ in \eqref{w3} as 
$$
\hlamn = BL \Kmm \; \max\left( \hC, 6\right) ,
$$
where $B$ is a universal constant. Its theoretical analog is denoted  by 
$$
\lamn = BL \Kmm \; \max\left( C, 6\right). 
$$
The form of $\lamn$ and $\hlamn$ are determined by the probabilistic inequality in Theorem \ref{th_random}. We will discuss it later.

Except the above assumptions we will need a few more  notations to state the main theorem. The first one is the oracle
$\ths$ that is defined as 
\begin{equation}
\label{oracle}
\ths = \arg \min_{\theta \in \Theta_1} \left\{(1+4 \delta)\left[Q(\fthe) - Q(f^0)\right] +
8\delta H \left(\frac{\lambda_n\sqrt{|S_\theta|}}{\delta A (S_\theta)}\right)
\right\},
\end{equation}
where $\delta $ is an arbitrary number in $\left(0,\frac{1}{4}\right)$  and the set $\Theta_1$ consists of such $\theta \in \Theta$  that the compatibility condition holds for 
$S_\theta$ (see Assumption 3). The oracle is the  element from the subclass of the family $\mF$ that is able, 
in the best possible way, to predict the ordering between objects and to select the sparse model simultaneously.  Furthermore, set
$S_*= S_{\ths}, A_* = A(S_*)$ and 
\begin{equation}
\label{error}
\varepsilon^* = (1+4\delta)\left[Q(\fths) - Q(f^0)\right] +
8\delta H \left(\frac{\lambda_n\sqrt{|S_*|}}{\delta A_*}\right).
\end{equation}
Now we can state the main result of the paper.
\begin{theorem}
\label{main}
Suppose that $\delta \in \left(0,\frac{1}{4}\right)$ is an arbitrary number and  Assumptions 1-4 are satisfied. 
Probability that at least one of the inequalities
\begin{equation}
\label{main1}
(1-4 \delta)\left[Q(\fhth) - Q(f^0)\right] + \hlamn |\hth - \ths|_1 \leq 
2 \varepsilon^* 
\end{equation}
or
\begin{equation}
\label{main2}
Q(\fhth) - Q(f^0) + \hlamn |\hth - \ths|_1 \leq 
4 \varepsilon^* 
\end{equation}
is satisfied is at least $1- \frac{3}{m^2} $.
\end{theorem}

Theorem \ref{main} states that the excess risk of the estimator can be  upper bounded, with
respect to the constants, by \eqref{error} which is the sum of two expressions. The first one 
$Q(\fths) - Q(f^0)$ is the approximation risk (the excess risk of the oracle) which describes how well the best rule 
$f^0$ can be approximated (with respect to the risk) by the sparse oracle $\fths. $ The second term in \eqref{error} is the bound for the estimation risk. It depends on the function $H$ from Assumption 2, the value $A_*$ from the compatibility condition and  the number of nonzero 
coordinates of the oracle. In the next section we give examples of interesting loss functions that satisfy Assumption 2  
with $G(u) = a u^{\frac{2}{\alpha}}$ for a positive constant $a$ and $\alpha \in (0,1].$
Notice that in these cases   
we would obtain the bound for the estimation risk that behaves like 
\begin{equation}
\label{ester}
\left(\frac{\log m \; |S_*|}{n}
\right)^{\frac{1}{2-\alpha}} \:,
\end{equation}
if we forget about dependence on values $A_*, C^2$ and other constants. Therefore, for sparse models we get oracle inequalities with fast rates, if $m$ does not grow exponentially fast with $n,$
that is a standard requirement for high-dimensional problems.  

Furthermore, analysing bounds for the excess risk of the estimator given in inequalities 
\eqref{main1} or \eqref{main2} we claim 
that the estimator behaves almost like the oracle that
 knows beforehand which parameters
(base functions) should be chosen to approximate appropriately the best function 
$f^0.$ Indeed, suppose that we knew a priori the support $S_*.$ Then we would estimate only  parameters contained in the set $S_*$, while setting others as zeros. The excess risk of such estimator can be bounded by the sum of the approximation risk and the estimation risk
\citep{rej:12}. Omitting constants, the latter is given by $H\left( \sqrt{\frac{|S_*|}{n}}\right),  $ that is simply  $\left(\frac{ |S_*|}{n}
\right)^{\frac{1}{2-\alpha}}$  for  $G(u) = a u^{\frac{2}{\alpha}}$ in Assumption 2.
Therefore, the term 
$\log m,$ that appears in \eqref{error} or \eqref{ester}, seems to be a price that we have to pay for not knowing a priori which parameters are in the  model.

Besides, Theorem \ref{main} gives the probabilistic inequality for the $l_1$ distance
$|\hth - \ths|_1$ between the estimator $\hth$ and the oracle $\ths$. This distance is small, if the best function $f^0$ can be well approximated by the sparse oracle with respect to the risk, i.e. the excess risk 
$Q(\fths) - Q(f^0)$ is small. In this case  the procedure estimates accurately parameters that are significant in the model (contained in the support $S_*$). But it is rather not able to discard all irrelevant 
coefficients. The remedy for that could be using the thresholded Lasso \citep{szhou:09} and 
the obtained bound for the distance $|\hth - \ths|_1$ would be useful in proving that supports of this 
modified estimator and the oracle coincide.

The formula of $\lamn$ is the consequence of Theorem \ref{th_random} that is a probabilistic inequality for a $U$-process and can be found in Appendix \ref{stochar}. Besides, there is a relation 
between $\lamn$ and the confidence level in Theorem \ref{main} that can be easily seen in proofs in Appendices. Roughly speaking, the confidence level can be increased, but it requires the 
growth of $\lamn,$ that makes the bound for the estimation risk worse. 
For simplicity, we choose the confidence level to be $1- \frac{3}{m^2}\:.$ 
Besides, the universal constant $B,$ that appears in formulas  $\lamn$ or $\hlamn,$ should be greater than  $998.$
Certainly this value can be decreased, but we have not attempted to optimize it nor other
constants in the paper. We have focused on rates of inequalities. Finally, it is worth to notice that the smoothing parameter $\hlamn$ can be calculated having the sample. In particular, it does not depend on unknown values $C^2$ nor $A_*.$ However, in practice the choice of $\hlamn$ based on  cross validation is recommended. 

Furthermore, Theorem \ref{main} is similar to theorems concerning properties of estimators obtained by minimization of an empirical process  with the Lasso penalty, especially those in \citet[Theorem 2.2]{geertar:06} and \citet[Theorem 6.4]{bulgeer:11}. Notice that the empirical process relates to the sum of independent random variables. Our problem is based on the $U$-process that relates to the sum of dependent random variables. Nevertheless, the results that we obtain in this more intricate scenario strictly correspond to those obtained for empirical processes. 
 
The proof of Theorem \ref{main} is divided into two parts that are contained in Appendices. The first one 
can be called ''stochastic''. We state a probabilistic inequality that bounds  $U$-processes in an appropriate way. It is based on quite simple but very helpful concentration inequality for $U$-processes. The second part of the proof is  ''deterministic''. We
adapt standard argumentation concerning empirical risk minimization with the Lasso 
penalty \citep{geertar:06, geer:08, bulgeer:11} to the ranking problem ($U$-processes).

\section{Discussion on assumptions}
\label{discussion}

The first and the last assumption  of Theorem \ref{main} are  standard. Thus, we focus on 
Assumption 2 and the compatibility condition that is involved in Assumption 3. Notice that they both depend on the choice of the pseudonorm $|| \cdot ||.$ We study two cases: 
the $\mathbb{L}^2$-pseudonorm $||f||_2 =\sqrt{\Ex f^2 (X,X')}$ and the ''conditional'' pseudonorm $||f||_c =\sqrt{\Ex \left[ \Ex_{X'} f(X,X')\right]^2},$ where 
$\Ex_{X'} f(x,X')= \Ex\left[f(X,X')|X=x \right]$ is the conditional expectation of the function $f.$
Obviously, for every function $f$ we have $||f||_c \leq ||f||_2.$ Therefore, if the condition \eqref{margin} is satisfied for $||\cdot ||_2,$ then it is also fulfilled with 
$||\cdot||_c$ provided that the function $G$ is nondecreasing. On the other hand, the condition \eqref{cc} with $||\cdot||_c$ implies the same one with $||\cdot||_2.$

\subsection{On compatibility condition}
\label{On_cc}

Consider the pseudonorm $||\cdot||_2.$ Let $\Psi(x,x')$ stands for the vector of base functions $\left[\psi_1(x,x'), \ldots, \psi_m(x,x')\right] ^T$ 
for $x,x' \in \mX.$ Besides, let a matrix $\Sigma = \Ex \Psi(X,X') \left[\Psi(X,X')\right] ^T$  
describe dependence between base functions. If the smallest eigenvalue $\rho$ of 
the matrix $\Sigma$ is positive, then it is well-known that the condition \eqref{cc} is satisfied for every set $S \subset \{1, \ldots ,m\}$ with the constant $\sqrt{\rho}.$ 
Thus, the compatibility condition, that looks similarly to the smallest eigenvalue requirement, is significantly weaker than that, because the compatibility condition  relates only to vectors satisfying  $|\theta_{S'}|_1 \leq 3 |\theta_{S}|_1$. 

Furthermore, notice that in the linear case, i.e. $f_\theta (x,x')= \theta^T(x-x'),$ 
the matrix $\Sigma$ is just twice the variance-covariance matrix $\Var (X)$ of the predictor $X.$ Moreover, in this  example the difference between $||f_\theta||_2$ and $||f_\theta||_c$ is rather negligible, because $||f_\theta||^2_2 = 2 \, \theta^T \Var (X) \theta$ and $||f_\theta||^2_c =  \theta^T \Var (X) \theta.$

\subsection{On Assumption 2}
\label{on_as2}

We have already mentioned that the convex conjugate $H$ of the function $G$ from Assumption 2 determines the rate of the bound of the estimation risk. We will show that distinct choices of  the pseudonorm $||\cdot ||$ can lead to significantly different rates. Similar argumentation, but related to the $0-1$ loss
function, can be found in \citep[Chapter 5]{clv:08}. 

Consider the hinge loss function 
$\phi(t) = \max \left(0, 1-t \right)$ and $f^0$ being the minimizer of the theoretical 
risk $\eqref{w2}$ among all measurable ranking rules $f:\mX \times \mX \rightarrow 
\mathbb{R}.$ It is not difficult to calculate that 
\begin{equation*}
f^0(x,x') =  \left\{
\begin{array}{rl}
1, & \quad {\rm if} \;    \eta(x,x') > \frac{1}{2},\\
-1, & \quad  {\rm otherwise,}
\end{array} \right.
\end{equation*}
where $\eta(x,x')=\Pro \left(Y>Y'| X=x,X'=x' \right) . $
To simplify calculations we restrict to  the standard linear model, that is
$
Y=\theta_0^T X + \varepsilon$ and $ Y'=\theta_0^T X' + \varepsilon '\:,
$
where $X,X' \sim N\left(0, V \right), 
\varepsilon, \varepsilon ' \sim N\left(0, \sigma^2 \right)$ and we consider a 
subclass $\mF _1= \left\{ \fthe \in \mF: |\fthe|_\infty \leq 1 \right\}.$ Therefore, 
we can calculate that  $\eta(x,x')= \Phi \left(\theta_0^T \left(x-x' \right)\right)$ and $\Phi$ is the distribution function of the standard normal variable. It implies that $f^0(x,x') = \sign \left(\theta_0^T(x-x') \right)$ and
$$
Q(\fthe) - Q(f^0) = \Ex \left|\fthe(X,X')  - f^0(X,X')\right| 
\left| 2 \eta(X,X')      -1 \right|.
$$
In the considered model it is not difficult to prove that for every $\alpha \in (0,1)$ there exists a positive $B(\alpha)$ such 
that for every $x \in \mX$ we have
\begin{equation}
\label{dom_norm1}
\Ex_{X'} \left|2 \eta(x,X')-1 \right|^{- \alpha} \leq B (\alpha).
\end{equation}
We will show that the condition \eqref{dom_norm1} implies, for every $\fthe \in \mF_1,$
the inequality 
\begin{equation}
\label{dom_norm2}
Q(\fthe) - Q(f^0)  \geq \left[B(\alpha) 2^{2-\alpha} \right]^{ - \frac{1}{\alpha}} \, ||\fthe -f^0||_c^{2/ \alpha} \:.  
\end{equation}  
Indeed, notice that the squared pseudonorm $||\fthe - f^0 ||_c ^2$ equals
\begin{equation}\nonumber 
\Ex_X \left[\Ex_{X'} \left(\fthe (X,X') - f^0 (X,X')\right)
\left|2 \eta(X,X')-1 \right|^{ \alpha/ 2} \left|2 \eta(X,X')-1 \right|^{- \alpha/ 2}
\right]^2\:,
\end{equation}
that is, by the Cauchy-Schwarz inequality and \eqref{dom_norm1}, upper bounded by
\begin{equation}
\label{dom_norm3}
B(\alpha) \; \Ex \left|\fthe (X,X') - f^0 (X,X')\right|^2
\left|2 \eta(X,X')-1 \right|^{ \alpha} \:.
\end{equation}
Moreover, we know that 
$$|\fthe (x,x') - f^0 (x,x')|^2 \leq 2^{2-\alpha} |\fthe (x,x') - f^0 (x,x')|^\alpha$$
for every $x,x'. $ Therefore, Jensen's inequality implies that the expression \eqref{dom_norm3} is not greater than
$$
B(\alpha) \, 2^{2-\alpha} \: \left[ \Ex \left|\fthe (X,X') - f^0 (X,X')\right|
\left|2 \eta(X,X')-1 \right| \right]^\alpha \:,
$$ 
which implies the inequality \eqref{dom_norm2}. Therefore, in the considered model  we can take the function $G$ in Assumption 2 arbitrarily close to the quadratic one. It implies that the function $H$ also can be arbitrarily close to the quadratic one. Thus, the bound of the estimation risk seems to have almost optimal rate.
The situation is different, if one considers the pseudonorm $||\cdot||_2$ instead of 
$||\cdot||_c \: .$ Using above assumptions on the model we can prove that 
the condition \eqref{margin} holds with the function $G(u)$ proportional to $u^4,$ so the function $H$ is propotional to $u^{4/3} \:.$
Indeed, it is enough to use \citet[Lemma 3.1]{geertar:06} and notice that their condition 
(20) holds in our model with the exponent one. Obviously, the rate of the estimation risk is worse than in the previous case.  To the best of our knowledge, it is rather not possible to 
bring this result much closer to the one  with the pseudonorm $||\cdot||_c\:.$ 

In the previous paragraph we describe a model that the stipulation \eqref{margin} is implied by the condition \eqref{dom_norm1} concerning the probabilistic property of the model.
However, in many practical examples convexity of the loss function $\phi$ 
(in the appropriate sense) is enough to guarantee the assumption \eqref{margin}. Indeed, consider 
the model with the pseudonorm $||\cdot||_2 $ and $\bF=\mF,$ that is $f^0$ is the function that 
minimizes the theoretical risk \eqref{w2} in the class $\mF.$ Suppose that the loss function 
$\phi$ is ''strictly convex'' in the sense that is described in \citet{m:02, bja:06} for the classification theory or \citet{rej:12} for the ranking problem. Omitting  details that
can be found in those papers we emphasize that many popular loss functions are strictly convex in that sense, for instance the exponential loss $\phi(t) = e^{-t},$ 
the truncated quadratic loss $\phi(t) =\left[ \max \left(0, 1-t \right)\right] ^2 $ or the logistic loss $\phi(t) = \ln \left(1+ e^{-t} \,\right),$ but the hinge loss 
$\phi(t)=\max(0,1-t)$ is not. In \citet[Lemma 5]{rej:12} one
proves that for such loss functions one obtains the condition \eqref{margin} with the 
quadratic function $G.$ Therefore, the convex conjugate $H$ is also the quadratic function.
Thus, the bound for the estimation risk seems to have the optimal rate. Obviously, the same result
also  holds for the pseudonorm $||\cdot||_c .$ Finally, notice that the first two of above-mentioned convex loss functions
do not satisfy the condition \eqref{lip} from Assumption 1. They 
fulfill the Lipschitz condition on every compact subset of the real line, but their Lipschitz 
constants depend on these subsets. It could make our argumentation in the proof of Theorem \ref{th_random} in Appendix \ref{stochar} slightly more difficult. Of course, this inconvenience disappears, if the family $\mF$ is uniformly 
bounded.   

\section{Conclusions}
\label{conclusions}

In the paper we study properties of estimators obtained by minimization of $U$-processes 
with the Lasso penalty. In the main theorem we obtained probabilistic inequalities concerning the quality of such estimators in prediction of the ordering between objects 
(via bounds for their excess risks) as well as model selection (via bounds for their 
$l_1$ distance from the oracle).

Besides, we should mention a possible improvement of results obtained in the paper, if 
different argumentation were used while proving the probabilistic inequality for a $U$-process in Appendix \ref{stochar}. Namely, one could  decompose a $U$-statistic into the sum of independent random variables and a degenerate $U$-statistic, and then bound these two terms separately. We do not proceed this way because of two reasons. First, following this argumentation we would  meet some technical problems, for instance the contraction principle that plays a crucial role in the proof of Theorem \ref{th_random} does not have an analogue for the Rademacher chaos of the order two that can be used to handle a degenerate $U$-process. The second reason is that this possible improvement would only replace the second moment of $\psi_k$ in the definition of $C^2$ by its conditional version. It would have only a slight impact on the obtained results. For instance, this improvement would influence negligibly on rates of oracle inequalities in Theorem \ref{main}, especially comparing to the impact of Assumption 2.

Finally, applying Theorem \ref{main} to the weighted Lasso penalty need not significant changes in argumentation \citep{geer:08, bulgeer:11}. The weighted Lasso is a frequently used modification of 
its initial form, i.e. we minimize 
\begin{equation*}
 Q_n(f_\theta) + \hlamn \sum\limits_{k=1}^m w_k |\theta_k|,
\end{equation*}
where $w_1, \ldots, w_m$ are (possibly random) nonnegative weigths. Inserting weights into the penalty allows to consider parameters that are, in researcher's opinion, definitely significant in the model by setting their weights as zeros. Moreover, taking 
\begin{equation}
\label{weight1}
w_k = \sqrt{ \frac{1}{n(n-1)} \sum_{i \neq j} \psi_k^2(X_i, X_j)
}
\end{equation}
leads to the model with normalized base functions that is often used in practice. 
For instance, in the linear ranking model, i.e. $\psi_k(x,x') = x_k - x_k',$  
the weight \eqref{weight1} is the empirical standard deviation
of the $k$-th predictor $X_k$ multiplied by $\sqrt{2}.$ Therefore, using such weights
 we work with base functions that are measured on the same scale. Finally, our results cannot be applied to popular two-step algorithms, for 
instance the adaptive Lasso penalty \citep{zou:06}. These procedures should be considered
individually.

\section*{Appendix}

The next two sections are devoted to the proof of Theorem \ref{main}. 

\appendix

\section{Stochastic arguments}
\label{stochar}

In this section we tackle the stochastic part of the problem. It can be divided into two components. The first one relates to obtaining probabilistic inequalities for  $U$-processes.

For an arbitrary number $M>0$ define a family
$$\mF_M = \{f_\theta \in \mF: |\theta - \ths|_1 \leq M \}
$$
and a $U$-process over a family $\mF_M$ 
$$U(M)= \sup_{ \fthe \in \mF_M} | Q_n(\fthe) - Q_n (\fths) - Q(\fthe) + Q(\fths) | .
$$
In this section we suppose that only Assumption 1 and 4 are satisfied. Henceforth we do not 
recall it.

\begin{theorem}
With probability at least 
\label{th_random}
$1- \frac{1}{m^2}$
\begin{equation*}
U(M) \leq 36 \sqrt{3}\, M L \sqrt{ \frac{\log m}{n}} \max(C,6).
\end{equation*}
\end{theorem}

In the proof of Theorem \ref{th_random} we use the following concentration inequality for $U$-processes which similar version can be found in \citet[Theorem 2]{rejchel:15}. 
We give its proof below to show the whole way of proving the main theorem that can be 
possibly  improved as we have mentioned in Section \ref{conclusions}.

\begin{theorem}
\label{th_conc}
Let $\mG$ be a subset of a family of functions $\{g:\mZ \times \mZ \rightarrow \mathbb{R}\}$
that are uniformly bounded by a constant $b>0$ and $\sigma^2 = \sup_{g \in \mG} \Var g(Z_1, Z_2).$ Let an arbitrary $U$-process over a family $\mG$ be given by
$$
U= \sup_{g \in \mG} \left|\frac{1}{n(n-1)} \sum_{i \neq j} g(Z_i, Z_j)  - \Ex g(Z_1,Z_2)\right|,
$$ 
and for $N= \left\lfloor \frac{n}{2} \right\rfloor$ we denote
$$
T = \sup_{g \in \mG} \left| \frac{1}{N} \sum_{i =1}^N  g (Z_i, Z_{N+i} ) - \Ex g(Z_1,Z_2)  \right|.
$$
Then for every $t>0$ with probability at least $1- \exp(-n t^2)$
$$
U \leq 2 \Ex T + \sqrt{6} \sigma t + 10 b t^2.
$$
\end{theorem}
The expression  $T$ in Theorem \ref{th_conc} is the supremum of an average of independent random 
variables. Thus, 
this theorem reduces investigating properties of  the
$U$-process $U$ to the significantly simpler empirical process $T.$ Similar exponential
inequalities for $U$-processes can be found in the literature, for instance
\citet[Theorem 5]{arc:95}. In fact, Theorem \ref{th_conc} is weaker than them, because it exploits the variance of functions in $\sigma^2$ instead of the conditional variance. However, the presented theorem is absolutely sufficient to obtain satisfactory bounds for $U(M)$ as we have discussed in Section \ref{conclusions}. Moreover, its application to our problem  is immediate.

\begin{proof}[Proof of Theorem \ref{th_conc}]
Fix $t,u, \lambda>0.$
From Markov's inequality we have
\begin{eqnarray}
\label{toemp1}
\Pro \left(U > \Ex T +\frac{u}{N} \right)&=& \Pro \left[ \exp\left( \lambda N U \right) >
\exp\left(\lambda N \Ex T  + \lambda u \right) \right] \nonumber \\
&\leq& \exp\left( -\lambda N \Ex T  - \lambda u  \right) \Ex \exp\left[ \lambda N U \right] .
\end{eqnarray}
We have mentioned in Section \ref{intro} that using Hoeffding's decomposition  we can represent every $U$-statistic as an
average of averages of independent random variables \citep{serf:80}, i.e.
\begin{equation}
\label{Hdecom}
\frac{1}{n(n-1)} \sum_{i \neq j} g(Z_i, Z_j)   =
\frac{1}{n!} \sum_\pi \frac{1}{N} \sum_{i=1}^N
 g \left(Z_{\pi(i)}, Z_{\pi(N+i)} \right),
\end{equation}  
where  the first sum  on the right-hand side of \eqref{Hdecom} is taken over all permutations $\pi$ of a set $\{1, \ldots, n\}.$ Applying \citep[Lemma A.1]{clv:08} we can  estimate 
$\Ex \exp\left[ \lambda N U \right]$
by $\Ex \exp\left(\lambda N T \right). $ Therefore, the right-hand side of \eqref{toemp1} is bounded by
\begin{equation}
\label{toemp2}
\exp( -\lambda u) \; \Ex \exp\left[ \lambda N (T- \Ex T ) \right].
\end{equation}
We have already mentioned that  $T$ is the supremum of an empirical process and concentration inequalities for  it  can be found in the literature \citep{talag:96, mass:00b, blm:00, bousquet:02}. We use \citet[Theorem 14.1]{bulgeer:11} and bound \eqref{toemp2} by
\begin{equation}
\label{toemp3}
\exp \left[-\lambda u  + (\exp(2\lambda b) -1 -2\lambda b)  \, v \right]
\end{equation}
with $v=\left(\frac{N \Ex T}{b}
+ \frac{N \sigma^2}{4 b^2} \right).$
Moreover, using argumentation from \citet[pages 281-282]{blm:00} we estimate 
 \eqref{toemp3} by
$$\exp\left[- \frac{3u^2}{4b (6bv+u )}   \right].$$
Therefore, we have
$$
\Pro \left(U > \Ex T +\frac{u}{N} \right) \leq \exp\left[- \frac{3u^2}{4b (6bv+u )}   \right],
$$
which easily implies that for every $x>0$
$$\Pro \left(  U \leq \Ex T + \sqrt{\frac{8bx}{N} \; \Ex T + \frac{2x \sigma^2}{N} }+ \frac{4bx}{3N}                   \right) \geq 1 - \exp(-x).
$$
To finish the proof of the theorem it is enough to take $x=nt^2$ and use two simple facts that $N \geq \frac{n}{3}$ and $2 \sqrt{ab} \leq a+b$ for all nonnegative numbers $a,b.$  

\end{proof}

\begin{proof}[Proof of Theorem \ref{th_random}]
We apply Theorem \ref{th_conc} with $\mG= \left\{ \phi_{\fthe} - \phi_{\fths}: \fthe \in \mF_M
\right\}.$ From the condition \eqref{lip} and Assumption 4 we get  
$|g|_\infty \leq ML \Km$ and 
$\Var g \leq M^2 L^2 C^2$ for every $g \in \mG.$
Thus, from Theorem \ref{th_conc} we obtain that for every $t>0$ with probability
 at least $1- \exp(-n t^2)$
\begin{equation}
\label{UMbound}
U(M) \leq 2 \Ex T(M) + \sqrt{6} M LC t + 10 M L \sqrt{\frac{n}{\log m}}\,  t^2\,,
\end{equation}
where
$$T(M)= \sup_{ \fthe \in \mF_M} \left|\frac{1}{N} \sum_{i =1}^N \left[ 
\phi_{\fthe} (Z_i, Z_{N+i} ) - \phi_{\fths} (Z_i, Z_{N+i} ) \right]
- Q(\fthe) + Q(\fths) \right|.
$$
The expression  $T(M)$ is again the supremum over an average of independent variables, so 
to bound its expectation we use standard methods from the empirical process theory.
We start with Symmetrization Lemma \citep[Lemma 2.3.1]{vw:96} that gives us
\begin{equation}
\label{symmetr}
\Ex T(M) \leq 2 \; \Ex \sup_{ \fthe \in \mF_M} \left|\frac{1}{N} \sum_{i =1}^N 
\varepsilon_i \left[ 
\phi_{\fthe} (Z_i, Z_{N+i} ) - \phi_{\fths} (Z_i, Z_{N+i} ) \right] \right|,
\end{equation}
where we have an additional sequence of iid random variables $\varepsilon _1, \ldots, \varepsilon _N $ (the Rademacher sequence). Variables $\varepsilon_i$'s take values $1$ or $-1$ with probability $\frac{1}{2}$ and are independent of the sample $Z_1, \ldots, Z _n .$
Next, by the contraction principle in \citet[Theorem 4.12]{ledtal:91} or 
\citet[Theorem 14.4]{bulgeer:11} the right-hand side of \eqref{symmetr} is bounded by
\begin{equation*}
 4 L \; \Ex \sup_{ \fthe \in \mF_M} \left|\frac{1}{N} \sum_{i =1}^N 
\varepsilon_i \left[ 
\fthe (X_i, X_{N+i} ) - \fths (X_i, X_{N+i} ) \right] \right|,
\end{equation*}
that can be estimated by
\begin{equation}
\label{contr2}
4ML \; \Ex  \max_{1\leq k \leq m} \left|\frac{1}{N} \sum_{i =1}^N 
\varepsilon_i  \psi_k(X_i, X_{N+i} ) \right|.
\end{equation}
Therefore, by Bernstein's inequality \citep[Lemma A.1]{geer:08} we bound \eqref{contr2} by
$
8 \sqrt{3} ML \, \sqrt{\frac{\log m}{n}} \left( C  + \sqrt{3}                       \right).
$
Finally, taking $t=\sqrt{2 \frac{\log m}{n}} $ and using the inequality  \eqref{UMbound} we receive that with probability at least 
$1- \frac{1}{m^2}$
$$
U(M) \leq 36 \sqrt{3} \,  M L  \sqrt{ \frac{\log m}{n}} \max(C,6).
$$
\end{proof}

The second task of this section is proving that  probability of the event
$\left\{\frac{\lambda_n}{2} \leq  \hlamn \leq 2 \lamn\right\}$ is close to one. It follows from the next lemma that 
is similar to \citet[Lemma 5.7]{geertar:06}. However, our estimator $\hC ^2$ is built on the basis of $U$-statistics
that requires  modified argumentation.

\begin{lemma} Each of the following two inequalities 
\label{l_random}
\begin{eqnarray}
\label{l_random1}
\frac{1}{2} \max(C, 6) >  \max( \hC, 6) ,\\
\label{l_random2}
2 \max(C, 6) <  \max( \hC, 6) .
\end{eqnarray}
holds with probability at most $\frac{1}{m^2}\:.$
\end{lemma}

\begin{proof}
We start with the inequality \eqref{l_random1}.
Similarly to the proof of \citet[Lemma 5.7]{geertar:06} we consider two cases. In the first one 
we assume that $C\leq 6.$ Obviously, we  have $\max(\hC, 6) \geq \max(C,6), $ 
that implies 
$$\Pro \left( \frac{1}{2} \max(C, 6) >  \max( \hC, 6) \right) = 0.$$
In the second case we suppose that $C>6$
and denote by $\psi_{max}$ the  base function with largest second moment, i.e. 
$
\Ex \psi_{max}^2 (X, X')= C^2.
$ 
Thus, probability of \eqref{l_random1} is upper bounded by 
$ \Pro \left( C > 2  \hC \right)
$
that is not greater than
\begin{equation}
\label{l_r2}
\Pro \left(\frac{1}{n(n-1)} \sum_{i \neq j} \psi_{max}^2 (X_i, X_j) -C^2 < -\frac{3}{4} C^2
\right).
\end{equation}
From Bernstein's inequality for $U$-statistics \citep[Theorem A, Section 5.6.1]{serf:80} we can estimate \eqref{l_r2} by  
$$
\exp \left(- \frac{\frac{9}{16} N C^4}{2 C^2 \frac{n}{\log m  }+ C^2 \frac{n}{2 \log m  } }
\right),
$$
which is not greater than 
$\exp(-2.7 \log m) ,
$
because $C>6.$ It finishes the proof of \eqref{l_random1}.

Next, we work with the inequality \eqref{l_random2}. It is trivial, if $\hC \leq 6, $ therefore we suppose that $\hC > 6. $ Then $\max(\hC, 6) =\hC$ and we have the following  
equality and  two inequalities that can be easily argued
\begin{eqnarray*}
\Pro \left( 2 \max(C, 6) <  \max( \hC, 6) \right) = \Pro \left( \hC^2 - C^2 >  3 \max(C^2,36) \right)\\
\leq \Pro \left(\max_k \frac{1}{n(n-1)} \sum_{i \neq j} \psi_{k}^2 (X_i, X_j) - \Ex \psi^2_k
> 3 \max(C^2,36) \right)\\
\leq \sum_{k=1}^m \Pro \left( \frac{1}{n(n-1)} \sum_{i \neq j} \psi_{k}^2 (X_i, X_j) - \Ex \psi^2_k
> 3 \max(C^2,36) \right).
\end{eqnarray*}
Using again Bernstein's inequality for $U$-statistics we can bound the last expression by
$$
m \exp \left( - \frac{9 n \max(C^4, 36^2)}{6 C^2 \frac{n}{\log m} + 6 \max(C^2,36) 
\frac{n}{\log m} } \right)\:,
$$
that is, by simple calculations, not greater than $\frac{1}{m^2}\:.$
\end{proof}

\section{Deterministic arguments}
\label{detar}

In this section we present the proof of Theorem \ref{main}. It is based on Theorem \ref{th_random} and Lemma \ref{l_random} from Appendix~\ref{stochar}  as well as standard argumentation concerning empirical risk minimization with the Lasso penalty \citep{geertar:06, geer:08, bulgeer:11}. However, we have to adjust the latter methods to the ranking problem.

\begin{proof}[Proof of Theorem \ref{main}]
Let $M^*=\frac{16 \varepsilon^*}{ \lamn} \:.$ 
From Theorem \ref{th_random} and Lemma \ref{l_random} we know that probability of 
the event 
\begin{equation}
\label{event}
\{ U(M^*) \leq \varepsilon_* \} \cap \left\{\frac{\lambda_n}{2} \leq  \hlamn \leq 2 \lambda_n \right\} 
\end{equation}
 is not less
than $1- \frac{3}{m^2}\:.$ Our argumentation in the rest of the proof is ''deterministic''. 
We start with defining $t= \frac{M^*}{M^* + |\hth - \ths|_1}$ and $\tth = t \hth +(1-t) \ths.$ 
Obviuosly, we have 
$
|\tth - \ths|_1\leq M^*,
$
so $\ftth \in \mF_{M^*},$ because $\Theta$ is a convex subset of $\mathbb{R}^m.$
Using convexity of the loss function $\phi$ we obtain the inequality
$$
Q_n(\ftth) \leq t Q_n(\fhth) +(1-t) Q_n(\fths),
$$
which implies that
\begin{equation}
\label{d_f1}
Q(\ftth) -Q(f^0) \leq 
U(M^*) + t[Q_n(\fhth) - Q_n(\fths)] + Q(\fths) -Q(f^0).
\end{equation}
Combining the inequality \eqref{d_f1} with  facts that we consider the event \eqref{event}  and $\fhth$ minimizes \eqref{w3} we obtain the inequality
\begin{equation*}
Q(\ftth) -Q(f^0) + \hlamn |\tth|_1\leq \varepsilon^* + \hlamn |\ths|_1 + Q(\fths) -Q(f^0).
\end{equation*}
Properties $|\theta|_1 = |\theta_{S_*}|_1 +  |\theta_{S'_*}|_1$  and
$\ths _{S'_*}=0$ imply two inequalities
\begin{equation}
\label{detform1}
Q(\ftth) -Q(f^0) + \hlamn |\tth_{S_*}|_1 + \hlamn |\tth_{S'_*}|_1
\leq \varepsilon^* +Q(\fths) -Q(f^0)+ \hlamn |\ths _{S_*}|_1 
\end{equation}
and 
\begin{equation}
\label{detform2}
Q(\ftth) -Q(f^0) + \hlamn |\tth_{S'_*}|_1 
\leq 2 \varepsilon^* + \hlamn |\tth_{S_*} - \ths _{S_*}|_1 .
\end{equation}
We will consider two cases. In the first one we assume that
$
\hlamn |\tth_{S_*} - \ths _{S_*}|_1 \leq \varepsilon^*,
$
which together with \eqref{detform2} and $\ths _{S'_*}=0$ implies that 
\begin{equation}
\label{d_f3}
Q(\ftth) -Q(f^0) + \hlamn |\tth_{S'_*} -\ths _{S'_*}|_1 \leq 3 \varepsilon^*.
\end{equation}
Adding $\hlamn |\tth_{S_*} - \ths _{S_*}|_1$ to both sides of \eqref{d_f3}  we get
\begin{equation}
\label{firstr}
Q(\ftth) -Q(f^0) + \hlamn |\tth -\ths|_1 \leq 4 \varepsilon^* ,
\end{equation}
so $ \hlamn |\tth -\ths|_1 \leq 4 \varepsilon^*  . $
Finally, on \eqref{event} we have  $\lamn \leq 2 \hlamn,$ therefore  we obtain that
$|\tth -\ths|_1 \leq \frac{M^*}{2},$ so $|\hth -\ths|_1 \leq M^*.$

Next, we consider the second case, i.e. we suppose that  $\hlamn|\tth_{S_*} - \ths _{S_*}|_1 > \varepsilon^*.$
Therefore, adding $\hlamn |\tth_{S_*} - \ths _{S_*}|_1$ to both sides of \eqref{detform1} and making simple calculations we obtain
\begin{equation}
\label{d_f4}
Q(\ftth) -Q(f^0) + \hlamn|\tth -\ths|_1 \leq  \varepsilon^* +Q(\fths) -Q(f^0)
+2\hlamn |\tth_{S_*} - \ths _{S_*}|_1 .
\end{equation}
Besides, in this case we have from \eqref{detform2} that
$
\hlamn |\tth_{S'_*}|_1  \leq 3\hlamn |\tth_{S_*} - \ths _{S_*}|_1 ,
$
so
$
|\tth_{S'_*} - \ths _{S'_*}|_1 \leq 3 |\tth_{S_*} - \ths _{S_*}|_1.
$
Therefore, we can use the compatibility condition for the set $S_*$ (by Assumption 3) and bound the right-hand side of 
\eqref{d_f4} by 
\begin{equation*}
\varepsilon^* +Q(\fths) -Q(f^0) + 2\hlamn \frac{||\ftth - \fths|| \sqrt{|S_*| }}{A_*} \:.
\end{equation*}
Using Assumption 2 as well as  facts that $uv \leq G(u) + H(v) $ for every $u,v$ and 
$\hlamn \leq 2 \lamn$ 
  we obtain the chain of inequalities
\begin{eqnarray*}
&\,& 2\hlamn \frac{||\ftth - \fths|| \sqrt{|S_*| }}{A_*} \leq 4\delta \frac{\lamn  \sqrt{|S_*| }}{\delta A_*}
||\ftth - f^0|| + 4\delta \frac{\lamn  \sqrt{|S_*| }}{\delta A_*}
||\fths - f^0||
\\
&\leq&8 \delta H\left( \frac{\lamn  \sqrt{|S_*| }}{\delta A_*}\right) + 4\delta G\left(
||\ftth - f^0|| \right)+ 4\delta G\left(
||\fths - f^0|| \right)
\\
&\leq& 8 \delta H\left( \frac{\lamn  \sqrt{|S_*| }}{\delta A_*}\right) + 4\delta 
\left[ Q(\ftth) -Q(f^0) \right]+4 \delta
\left[ Q(\fths) -Q(f^0) \right].
\end{eqnarray*}
Summarizing,   we can upper bound the left-hand side of the inequality \eqref{d_f4} 
 by
$$
 \varepsilon^* +(1+4\delta)
\left[Q(\fths) -Q(f^0)\right]+
8 \delta H\left( \frac{\lamn  \sqrt{|S_*| }}{\delta A_*}\right) + 4\delta 
\left[ Q(\ftth) -Q(f^0) \right],
$$
that gives us 
\begin{equation}
\label{detform4}
(1-4\delta) \left[Q(\ftth) -Q(f^0)\right] + \hlamn |\tth -\ths|_1 \leq  2 \varepsilon^*.
\end{equation}
The inequality \eqref{detform4} implies that $ \hlamn |\tth -\ths|_1 \leq 2 \varepsilon^*,$ so
$|\tth -\ths|_1 \leq \frac{M^*}{4}.$ 

Thus, in both cases we obtain that  $|\hth -\ths|_1 \leq M^*$ or to be more precise
$$\Pro \left( |\hth -\ths|_1 \leq M^* \right) \geq  1- \frac{3}{m^2} \:.$$ 
To finish the proof of Theorem \ref{main} it is enough to repeat above argumentation  with $\hth$ instead  of $\tth $ that leads to \eqref{firstr} or \eqref{detform4} with
$\tth$ replaced by  $\hth .$ 
\end{proof}

{\bf Acknowledgements}
Research financed by Polish National Science Centre no. DEC-2014/12/S/ST1/00344.

\bibliographystyle{apalike} 
\bibliography{ULasso}   

\begin{thebibliography}{}

\bibitem[Agarwal et~al., 2005]{aghhr:05}
Agarwal, S., Graepel, T., Herbrich, R., Har-{P}eled, S., and Roth, D. (2005).
\newblock Generalization bounds for the area under the {ROC} curve.
\newblock {\em J. Machine Learning Research}, 6:393--425.

\bibitem[Arcones, 1995]{arc:95}
Arcones, M.~A. (1995).
\newblock A {B}ernstein-type inequality for {$U$}-statistics and
  {$U$}-processes.
\newblock {\em Statist. Probab. Letter}, 22:223--230.

\bibitem[Bartlett et~al., 2005]{bbm:05}
Bartlett, P.~L., Bousquet, O., and Mendelson, S. (2005).
\newblock Local rademacher complexities.
\newblock {\em Ann. Statist.}, 33:1497--1537.

\bibitem[Bartlett et~al., 2006]{bja:06}
Bartlett, P.~L., Jordan, M.~I., and McAuliffe, J.~D. (2006).
\newblock Convexity, classification and risk bounds.
\newblock {\em Journal of the American Statistical Association}, 101:138--156.

\bibitem[Bickel et~al., 2009]{bickel:09}
Bickel, P.~J., Ritov, Y., and Tsybakov, A.~B. (2009).
\newblock Simultaneous analysis of {L}asso and {D}antzig selector.
\newblock {\em Ann. Statist.}, 37:1705--1732.

\bibitem[Bobrowski and {\L}ukaszuk, 2012]{bobluk:12}
Bobrowski, L. and {\L}ukaszuk, T. (2012).
\newblock Prognostic modeling with high dimensional and censored data.
\newblock {\em Lecture Notes in Computer Science}, 7377:178--193.

\bibitem[Boucheron et~al., 2000]{blm:00}
Boucheron, S., Lugosi, G., and Massart, P. (2000).
\newblock A sharp concentration inequality with applications.
\newblock {\em Random Structures Algorithms}, 16:277--292.

\bibitem[Bousquet, 2002]{bousquet:02}
Bousquet, O. (2002).
\newblock A {B}ennett concentration inequality and its application to suprema
  of empirical processes.
\newblock {\em C. R. Acad. Sci. Paris}, Ser. I 334:495--500.

\bibitem[B\"{u}hlmann and van~de Geer, 2011]{bulgeer:11}
B\"{u}hlmann, P. and van~de Geer, S. (2011).
\newblock {\em Statistics for High-Dimensional Data: Methods, Theory and
  Applications}.
\newblock Springer, New York.

\bibitem[Bunea et~al., 2007]{buneatw:07}
Bunea, F., Tsybakov, A., and Wegkamp, M.~H. (2007).
\newblock Sparse density estimation with $l_1$ penalties.
\newblock {\em COLT}, 4539:530--543.

\bibitem[Chen and Wu, 2013]{chenwu:13}
Chen, H. and Wu, J. (2013).
\newblock Regularized ranking with convex losses and $l_1$-penalty.
\newblock {\em {A}bstract and {A}pplied {A}nalysis}, 2013.

\bibitem[Cl{\'{e}}men\c{c}on et~al., 2008]{clv:08}
Cl{\'{e}}men\c{c}on, S., Lugosi, G., and Vayatis, N. (2008).
\newblock Ranking and empirical minimization of {$U$}-statistics.
\newblock {\em Ann. Statist.}, 36:844--874.

\bibitem[Cortes and Mohri, 2004]{cortesmoh:04}
Cortes, C. and Mohri, M. (2004).
\newblock {AUC} optimization vs. error rate minimization.
\newblock In {\em {I}n {A}dvances in {N}eural {I}nformation {P}rocessing
  {S}ystems}, volume~16, pages 1436--1462.

\bibitem[Cossock and Zhang, 2006]{cz:06}
Cossock, D. and Zhang, T. (2006).
\newblock Subset ranking using regression.
\newblock In {\em {P}roceedings of the 19th {A}nnual {C}onference on {L}earning
  {T}heory}.

\bibitem[de~la Pe\~{n}a and Gin{\'{e}}, 1999]{pg:99}
de~la Pe\~{n}a, V.~H. and Gin{\'{e}}, E. (1999).
\newblock {\em Decoupling: From Dependence to Independence}.
\newblock Springer-Verlag, New York.

\bibitem[Freund et~al., 2004]{fss:04}
Freund, Y., Iyer, R., Schapire, R.~E., and Singer, Y. (2004).
\newblock An efficient boosting algorithm for combining preferences.
\newblock {\em J. Machine Learning Research}, 4:933--969.

\bibitem[Greenshtein, 2006]{greenshtein:06}
Greenshtein, E. (2006).
\newblock Best subset selection, persistence in high dimensional statistical
  learning and optimization under $l_1$ constraint.
\newblock {\em Ann. Statist}, 34:2367--2386.

\bibitem[Han, 1987]{han:87}
Han, A.~K. (1987).
\newblock Non-parametric analysis of a generalized regression model.
\newblock {\em Journal of Econometrics}, 35:303--316.

\bibitem[Huang and Zhang, 2012]{huang:12}
Huang, J. and Zhang, C.-H. (2012).
\newblock Estimation and selection via absolute penalized convex minimization
  and its multistage adaptive applications.
\newblock {\em Journal of Machine Learning Research}, 13:1839--1864.

\bibitem[Klein and Spady, 1993]{kleinsp:92}
Klein, R.~W. and Spady, R.~H. (1993).
\newblock An efficient semiparametric estimator for binary response models.
\newblock {\em Econometrica}, 61:387--421.

\bibitem[Lai et~al., 2013]{lai:13}
Lai, H., Pan, Y., Liu, C., Lin, L., and Wu, J. (2013).
\newblock Sparse learning-to-rank via an efficient primal-dual algorithm.
\newblock {\em IEEE Trans. Comput.}, 62:1221--1233.

\bibitem[Laporte et~al., 2014]{laporte:14}
Laporte, L., Flamary, R., Canu, S., D{\'{e}}jean, S., and Mothe, J. (2014).
\newblock Non-convex regularizations for feature selection in ranking with
  sparse {SVM}.
\newblock {\em IEEE Transactions on Neural Networks and Learning Systems},
  25:1118--1130.

\bibitem[Ledoux and Talagrand, 1991]{ledtal:91}
Ledoux, M. and Talagrand, M. (1991).
\newblock {\em Probability in Banach Spaces: Isoperimetry and Processes}.
\newblock Springer, Berlin.

\bibitem[Mammen and Tsybakov, 1999]{mamtsyb:99}
Mammen, E. and Tsybakov, A.~B. (1999).
\newblock Smooth discrimination analysis.
\newblock {\em Ann. Statist.}, 27:1808--1829.

\bibitem[Massart, 2000]{mass:00b}
Massart, P. (2000).
\newblock About the constants in {T}alagrand's inequality for empirical
  processes.
\newblock {\em Ann. Probab.}, 29:863--884.

\bibitem[Meinshausen and Buhlmann, 2006]{meinbul:06}
Meinshausen, N. and Buhlmann, P. (2006).
\newblock High-dimensional graphs and variable selection with the lasso.
\newblock {\em Ann. Statist}, 34:1436--1462.

\bibitem[Mendelson, 2002]{m:02}
Mendelson, S. (2002).
\newblock Improving the sample complexity using global data.
\newblock {\em IEEE Trans. Inform. Theory}, 48:1977--1991.

\bibitem[Negahban et~al., 2012]{nega:12}
Negahban, S.~N., Ravikumar, P., Wainwright, M.~J., and Yu, B. (2012).
\newblock A unified framework for high-dimensional analysis of {$M$}-estimators
  with decomposable regularizers.
\newblock {\em Statistical Science}, 27:538--557.

\bibitem[Rejchel, 2012]{rej:12}
Rejchel, W. (2012).
\newblock On ranking and generalization bounds.
\newblock {\em Journal of Machine Learning Research}, 13:1373--1392.

\bibitem[Rejchel, 2015]{rejchel:15}
Rejchel, W. (2015).
\newblock Fast rates for ranking with large families.
\newblock {\em Neurocomputing}, page
  http://dx.doi.org/10.1016/j.neucom.2015.05.013i.

\bibitem[Rudin, 2006]{rudin:06}
Rudin, C. (2006).
\newblock Ranking with a $p$-norm push.
\newblock In {\em {I}n {P}roceedings of the 19th {A}nnual {C}onference on
  {L}earning {T}heory}.

\bibitem[Serfling, 1980]{serf:80}
Serfling, R.~J. (1980).
\newblock {\em Approximation {T}heorems of {M}athematical {S}tatistics}.
\newblock Wiley, New York.

\bibitem[Sherman, 1993]{sherman:93}
Sherman, R.~P. (1993).
\newblock The limiting distributions of the maximum rank correlation estimator.
\newblock {\em Econometrica}, 61:123--137.

\bibitem[Talagrand, 1996]{talag:96}
Talagrand, M. (1996).
\newblock New concentration inequalities in product spaces.
\newblock {\em Invent. Math.}, 126:503--563.

\bibitem[Tibshirani, 1996]{tib:96}
Tibshirani, R. (1996).
\newblock Regression shrinkage and selection via the lasso.
\newblock {\em Journal of the Royal Statistical Society, Series B},
  58:267--288.

\bibitem[van~de Geer, 2008]{geer:08}
van~de Geer, S. (2008).
\newblock High-dimensional generalized linear models and the lasso.
\newblock {\em Annals of Statistics}, 36:614--645.

\bibitem[van~de Geer and Tarigan, 2006]{geertar:06}
van~de Geer, S. and Tarigan, B. (2006).
\newblock Classifiers of support vector machine type with $l_1$ complexity
  regularization.
\newblock {\em Bernoulli}, 12:1045--1076.

\bibitem[van~der Vaart and Wellner, 1996]{vw:96}
van~der Vaart, A.~W. and Wellner, J.~A. (1996).
\newblock {\em {W}eak Convergence and Empirical Processes: With Applications to
  Statistics}.
\newblock Springer Verlag, New York.

\bibitem[Zhao and Yu, 2006]{zhaoyu:06}
Zhao, P. and Yu, B. (2006).
\newblock On model selection consistency of lasso.
\newblock {\em Journal of Machine Learning Research}, 7:2541--2563.

\bibitem[Zhou, 2009]{szhou:09}
Zhou, S. (2009).
\newblock Thresholding procedures for high dimensional variable selection and
  statistical estimation.
\newblock {\em Advances in Neural Information Processing Systems},
  22:1436--1462.

\bibitem[Zou, 2006]{zou:06}
Zou, H. (2006).
\newblock The adaptive lasso and its oracle properties.
\newblock {\em Journal of the American Statistical Association},
  101:1418--1429.

\end{thebibliography}

%
%

\end{document}